\documentclass[12pt]{article}

\usepackage[margin=1.2in]{geometry}
\usepackage{amsmath,amsthm,amssymb,mathtools}
\usepackage{natbib}
\usepackage{hyperref}
\usepackage{setspace}
\usepackage{booktabs}
\usepackage{graphicx}
\usepackage{xcolor}
\usepackage{caption}
\usepackage{subcaption}
\newtheorem{proposition}{Proposition}
\newtheorem{lemma}{Lemma}

\newtheorem{assumption}{Assumption}
\newtheorem{definition}{Definition}
\newtheorem{remark}{Remark}

\newcommand{\E}{\mathbb{E}}
\newcommand{\R}{\mathbb{R}}

\newcommand{\hpi}{\hat{\pi}}

\newcommand{\bA}{\bar{A}}
\newcommand{\bK}{\bar{K}}
\newcommand{\bB}{\bar{B}}
\newcommand{\diff}{\,\mathrm{d}}

\begin{document}

\title{\textbf{Cognitive Debt: AI as Intellectual Leverage \\ and the Dynamics of Systemic Fragility}}

\author{Shuchen Meng\\
\small New York University\\
\small \texttt{sm13533@nyu.edu}}

\date{June 2026 \\ \small \textit{Preliminary --- Comments Welcome}}

\maketitle

\begin{abstract}
\noindent
We develop a formal theory of \emph{cognitive debt}: the stock of unverified
reasoning obligations that accumulates when individuals use AI as a
substitute---rather than a complement---for first-principles cognition.
The model features two state variables per agent---\emph{cognitive capital}
(the unaided ability to reason, verify, and transfer knowledge) and
\emph{cognitive debt} (the corresponding unserviced obligation)---and
a multiplicative production technology in which cognitive capital functions
as collateral that determines the return to AI adoption.
We establish six propositions. (1) Rational agents incur positive cognitive
debt because the costs are deferred, partially external, and masked by short-run
productivity gains. (2) Tranquil periods lower subjective risk assessments, raise
AI substitution intensity, and compound leverage---generating a \emph{cognitive
Minsky moment}: subjective risk falls while true systemic fragility rises.
(3) Expected crisis losses are convex in the aggregate leverage ratio, so high-leverage
economies are disproportionately fragile. (4) Post-crisis, output-target pressure
produces a \emph{false-correction loop}: agents rationally patch AI failures with
more AI, ratcheting leverage upward and increasing the severity of future crises under competitive output pressure. (5) The
decentralised equilibrium over-adopts substitutive AI relative to the social
optimum, owing to three unpriced externalities: systemic risk, cognitive public
goods, and an arms-race externality. The optimal Pigouvian instrument is an
AI-use tax indexed to the aggregate leverage ratio. (6) In a
two-type heterogeneous-agent economy, high-cognitive-capital agents adopt AI
more intensively, erode their capital faster, and---in a reversal of fortune---eventually
hold \emph{less} unaided cognitive capital than initially lower-skilled agents;
heterogeneity also amplifies aggregate systemic risk via the convexity of crisis losses.
\smallskip


\end{abstract}

\newpage

\section{Introduction}
\label{sec:intro}

\paragraph{A productivity paradox.}
Consider two findings from the same experimental programme.
\citet{NoyZhang2023} randomly assign mid-career professionals access to a large
language model and find that treated workers complete writing tasks 37\% faster,
at quality rated higher by independent evaluators.
\citet{DellAcquaEtAl2023}, in a field experiment with Boston Consulting Group
analysts, replicate the productivity gain for routine in-distribution tasks---and
then ask workers to solve problems beyond the AI's training frontier.
For these \emph{out-of-distribution} tasks, the high-AI-use group
performs \emph{worse} than the control group, consistent with overreliance
and miscalibration at the AI capability frontier. We interpret this as a
short-run analogue of the longer-run cognitive-capital mechanism developed below.

The pattern is not an anomaly. It is the signature of a structural mechanism that
this paper formalises: short-run productivity and long-run cognitive capacity are
in tension, and agents who rationally maximise the former systematically underinvest
in the latter. The result is a hidden stock of obligations---the difference between
what agents can produce \emph{with} AI and what they can produce \emph{without} it.
We call this stock \emph{cognitive debt}.

\paragraph{This paper.}
We develop a dynamic equilibrium model of AI adoption with aggregate cognitive
externalities, addressing cognitive debt accumulation, systemic fragility, and welfare. The framework has two distinguishing features.
First, cognitive capital functions as the \emph{collateral} underlying AI adoption:
the marginal product of AI is strictly proportional to unaided cognitive capacity,
so capital erosion feeds back into rising AI dependence in a self-reinforcing cycle.
Second, AI failure is a \emph{common-mode} event: when AI systems share training
data and users share workflows, errors are correlated, and model collapse
\citep{ShumailovEtAl2024} degrades AI quality in proportion to aggregate reliance.
These two features together generate Minsky-style instability in the cognitive domain.

Our central contribution is to show that three individually benign features---which
each appear either desirable or at worst neutral---combine to produce systemic fragility.

\begin{enumerate}
\item \textbf{Multiplicative complementarity.} Cognitive capital is not merely
complementary to AI; it is the \emph{base asset} whose level determines the return
to AI use. High-capital agents extract more from AI; low-capital agents cannot
effectively prompt, verify, or transfer AI outputs. This creates a feedback: capital
erosion reduces the marginal product of AI, which does not reduce AI reliance---it
increases it, because output targets remain fixed.

\item \textbf{Invisible debt.} Cognitive debt is unobservable in normal operation.
Skill atrophy surfaces only when AI is unavailable, when tasks exceed the AI's
capability frontier, or when AI produces confident hallucinations that require
first-principles correction. In the model, this corresponds to the stress state:
debt exposure $\kappa z b_{it}$ is visible only at shock intensity $z > 0$.

\item \textbf{Systemic correlation.} AI errors are not idiosyncratic. When agents
adopt the same AI systems, their errors become correlated. When AI-generated content
dominates training pipelines, model collapse \citep{ShumailovEtAl2024} erodes AI
quality over time. Both forces mean that AI failure is a common-mode event, not a
diversifiable risk.
\end{enumerate}

\paragraph{Main results.} In a model with a continuum of agents choosing AI
substitution intensity and deliberate-practice investment, we establish:

\textbf{Proposition~\ref{prop:individual}} shows that rational agents optimally
incur positive cognitive debt. The determinants of equilibrium debt are transparent
and empirically observable: AI quality, output pressure, discount factors, the
subjective probability of AI failure, and the individual's exposure to failure costs.

\textbf{Proposition~\ref{prop:minsky}} establishes the cognitive Minsky moment.
Every tranquil period---a stretch without observed AI failures---updates subjective
risk downward and raises equilibrium AI substitution intensity. Cognitive leverage
$\Omega_t = \bar{B}_t / \bar{K}_t$ rises monotonically while the true crisis
probability also rises. The Minsky divergence---$\hat{\pi}_t \downarrow$ while
$\pi_t \uparrow$---is a necessary consequence of rational Bayesian updating in a
system where the true risk process is endogenous to the belief-driven leverage cycle.

\textbf{Proposition~\ref{prop:convex}} shows that crisis losses are convex in
leverage. Small increases in $\Omega$ near the fragility threshold produce
disproportionately large expected losses, because both the probability and the
severity of a cognitive crisis are increasing in $\Omega$.

\textbf{Proposition~\ref{prop:fcl}} characterises the \emph{false-correction loop}.
When the shadow price of current output exceeds the shadow price of future cognitive
capital recovery, the individually optimal response to an AI failure is to adopt
\emph{more} AI assistance---patching AI errors with AI. The condition is generically
satisfied under competitive output pressure, generating ratchet dynamics in leverage.

\textbf{Proposition~\ref{prop:welfare}} derives the social planner's optimum and
quantifies the three externalities that drive the decentralised gap: (i) a systemic
risk externality, because my AI use raises the aggregate leverage ratio and the
common-mode failure probability without my bearing the full loss; (ii) a cognitive
public-goods externality, because societal verification capacity depends on the
aggregate stock of unaided cognitive capital; and (iii) an arms-race externality,
because rising average AI use shifts the competitive output benchmark upward,
compelling others to increase AI reliance.

\textbf{Proposition~\ref{prop:hetero}} (heterogeneous agents) establishes a
\emph{reversal of fortune}: high-cognitive-capital agents adopt AI more intensively
(since the return to AI is proportional to $k$), erode their capital faster, and
eventually end up with \emph{less} unaided cognitive capital than low-capital agents
who adopted AI conservatively. In the long run, AI adoption compresses the
distribution of cognitive capital even as it initially widens output inequality.

Figure~\ref{fig:minsky} illustrates the Minsky divergence and leverage dynamics;
Figure~\ref{fig:phase} shows the equilibrium trajectory through the Hedge--Speculative--Ponzi
phase space; Figure~\ref{fig:hetero} displays the heterogeneous-agent inequality reversal.

\paragraph{Related literature.} Our paper lies at the intersection of four strands.

\textit{Technology and human capital.} \citet{AcemogluRestrepo2018} study automation
as task displacement; \citet{AcemogluRestrepo2019} introduce reinstatement tasks.
We differ in focusing on the \emph{intertemporal} erosion of the capital stock
rather than the contemporaneous task allocation, and in introducing a debt state
variable with its own compounding dynamics.

\textit{Financial fragility.} The Minsky framework \citep{Minsky1986} classifies
financial positions by the relationship between income flows and debt obligations.
\citet{KiyotakiMoore1997} formalise collateral constraints in credit cycles.
\citet{BernankeGertlerGilchrist1999} develop the financial accelerator.
We import the Minsky classification into the cognitive domain, with cognitive capital
playing the role of collateral and cognitive debt compounding through habit formation
and rising switching costs.

\textit{AI and productivity.} \citet{BrynjolfssonEtAl2025} and \citet{NoyZhang2023}
document short-run productivity gains from AI access. Crucially, \citet{DellAcquaEtAl2023}
identify the \emph{jagged frontier}: AI improves performance on in-distribution tasks
but \emph{reduces} it on out-of-distribution tasks for high-AI-reliance workers.
This asymmetry is the empirical anchor for our stress-state production function.

\textit{Systemic risk and externalities.} \citet{AllenGale2000} and
\citet{AcemogluEtAl2012} study contagion in financial networks. Model collapse
\citep{ShumailovEtAl2024} is the cognitive analogue: AI-generated content in training
pipelines degrades AI quality in a manner that is correlated across all users of a
given model family.

\paragraph{Organisation.} Section~\ref{sec:model} presents the model.
Section~\ref{sec:individual} solves the individual problem (Proposition~\ref{prop:individual}).
Section~\ref{sec:aggregate} characterises aggregate dynamics (Propositions~\ref{prop:minsky}
and \ref{prop:convex}).
Section~\ref{sec:fcl} analyses post-crisis dynamics (Proposition~\ref{prop:fcl}).
Section~\ref{sec:welfare} derives the social optimum and optimal policy
(Proposition~\ref{prop:welfare}).
Section~\ref{sec:hetero} develops the two-type heterogeneous-agent economy
(Proposition~\ref{prop:hetero}) and the reversal of fortune.
Section~\ref{sec:discussion} discusses extensions and limitations.
Section~\ref{sec:conclusion} concludes.
All proofs are in the Appendix.

\section{The Model}
\label{sec:model}

\subsection{Environment}

Time is discrete, $t = 0, 1, 2, \ldots$\;.
There is a unit continuum of agents indexed by $i \in [0,1]$.
Each agent is characterised by two state variables:
\begin{itemize}
\item $k_{it} \geq 0$: \emph{cognitive capital}---the unaided ability to reason,
verify, synthesise, and transfer knowledge. This is the stock of intellectual
competence that functions without AI support.
\item $b_{it} \geq 0$: \emph{cognitive debt}---the gap between current task demands
and current cognitive capital, made viable by AI. Formally, $b_{it}$ is the stock of
unverified reasoning obligations: outputs produced with AI assistance that the agent
could not independently reproduce, verify, or correct.
\end{itemize}
Each period, agent $i$ chooses two controls:
\begin{itemize}
\item $a_{it} \in [0,1]$: \emph{AI substitution intensity}---the fraction of
cognitively demanding steps outsourced to AI. We reserve the label \emph{substitutive}
AI for $a_{it} > 0$; AI used purely as a tutor, verifier, or feedback device that
requires the agent to engage first-principles reasoning corresponds to $a_{it} = 0$.
\item $x_{it} \geq 0$: \emph{deliberate-practice investment}---cognitive debt
repayment: hand-derivation, unaided testing, active verification, first-principles
reconstruction.
\end{itemize}

\subsubsection*{Competitive output benchmark}

Agents operate in a competitive environment where output is evaluated against an
endogenous aggregate benchmark. Let $\bar{y}_t$ denote the symmetric-equilibrium
average output. The per-period payoff for agent $i$ in state $s$ is:
\begin{equation}
u_{it}^s = y_{it}^s - \chi\,\max\bigl\{0,\; \bar{y}_t - y_{it}^s\bigr\},
\label{eq:benchmark}
\end{equation}
where $\chi \geq 0$ is a competitive pressure parameter capturing wage penalties,
labour-market displacement, or organisational ranking from falling below the
benchmark. Each agent takes $\bar{y}_t$ as given; the social planner internalises
its endogeneity. In the symmetric equilibrium $y_{it}^s = \bar{y}_t$, so the
penalty is zero in realisation; but the threat of falling below the benchmark when
others maintain high $A_t$ creates a strategic complementarity that drives
over-adoption at the social level (Section~\ref{sec:welfare}). The output constraint
$y_{it} \geq w$ in the post-crisis analysis (Section~\ref{sec:fcl}) is the
realised form of this benchmark when competitive pressure is acute.

\subsection{Production Technology}
\label{sec:production}

Nature draws a state $s_{t} \in \{N, S\}$ each period, where $N$ is the normal state
and $S$ is the stress state. The true probability of the stress state is $\pi_t$,
which is determined endogenously in equilibrium (Section~\ref{sec:aggregate}).
Agent $i$ holds the subjective belief $\hpi_{t}$ about the probability of the stress
state, which is updated via Bayes' rule as described below.

\subsubsection*{Normal-state production}

In state $N$, output is
\begin{equation}
y_{it}^{N} = k_{it} \cdot G(a_{it};\, q_t),
\label{eq:yN}
\end{equation}
where $q_t > 0$ denotes the quality of the prevailing AI system.

\begin{assumption}[Normal-state technology]
\label{ass:G}
$G : [0,1] \times \R_{++} \to [1,\infty)$ is twice continuously differentiable
and satisfies:
\begin{enumerate}
\item[\emph{(i)}] $G(0;\, q) = 1$ for all $q > 0$.
\item[\emph{(ii)}] $G_a(a;\, q) > 0$ and $G_{aa}(a;\, q) < 0$ for all $(a, q)$.
\item[\emph{(iii)}] $G_q(a;\, q) > 0$ and $G_{aq}(a;\, q) > 0$ for all $(a, q)$.
\item[\emph{(iv)}] $\lim_{a \to 0} G_a(a;\, q) = +\infty$ \emph{(Inada)}.
\end{enumerate}
\end{assumption}

Condition (i) normalises: without AI, output equals cognitive capital $k_{it}$.
Condition (ii) says AI raises output at a diminishing marginal rate. Condition (iii)
says higher AI quality raises both the level and the marginal product of AI use.
Condition (iv) ensures an interior solution whenever the cost of AI use is finite.

\begin{remark}[Cognitive capital as collateral]
\label{rem:collateral}
The multiplicative structure $y^N = k_{it} \cdot G(a_{it};\, q_t)$ encodes the key
mechanism of the paper. The marginal product of AI is
$\partial y^N / \partial a = k_{it} \cdot G_a(a_{it};\, q_t)$,
which is strictly proportional to $k_{it}$. Thus cognitive capital is not merely
an additive complement to AI; it is the \emph{scale factor}---the base asset---that
determines how much value an agent extracts from any given level of AI substitution.
An agent with $k = 0$ obtains zero output regardless of $a$: she cannot prompt
effectively, cannot evaluate responses, cannot transfer AI-generated results to novel
tasks. This multiplicative complementarity generates the leverage dynamics central to
the model.
\end{remark}

\subsubsection*{Stress-state production}

In state $S$---an AI outage, a hallucination in a critical task, an out-of-distribution
problem, an audit requiring first-principles justification, or any event that forces
the agent to operate at or beyond the AI's capability frontier---output is
\begin{equation}
y_{it}^{S} = k_{it} \cdot \tilde{G}(a_{it};\, q_t,\, z_t) - \kappa\, z_t\, b_{it},
\label{eq:yS}
\end{equation}
where $z_t > 0$ is the intensity of the stress shock drawn from a distribution with
support $[\underline{z}, \bar{z}]$, and $\kappa > 0$ is the debt-exposure coefficient.

\begin{assumption}[Stress-state technology]
\label{ass:Gtilde}
$\tilde{G}(a;\, q, z) = G\bigl(a;\, q \cdot s(z)\bigr)$, where
$s : \R_{++} \to (0, 1]$ satisfies $s(0^+) = 1$, $s'(z) < 0$, and $s(z) \to 0$
as $z \to \infty$.
\end{assumption}

Assumption~\ref{ass:Gtilde} says the stress state reduces \emph{effective AI quality}
by a factor $s(z) \in (0,1)$: the AI system is less reliable, its outputs require
deeper verification, or the task has moved outside the training distribution. Under this
specification, $\tilde{G}_a = G_a(\cdot;\, q s(z)) < G_a(\cdot;\, q)$, so AI assistance
provides a strictly smaller marginal benefit in the stress state.

The term $\kappa z_t b_{it}$ is the \emph{debt-service cost}: cognitive debt $b_{it}$
represents the accumulated gap between task demands and unaided cognitive capacity.
Under stress, this gap is exposed. Agents who have outsourced their verification,
derivation, and error-correction to AI cannot fulfil these obligations from their
own cognitive resources; the resulting output shortfall is proportional to $b_{it}$
and scaled by shock severity $z_t$.

\begin{remark}[Jagged frontier]
\citet{DellAcquaEtAl2023} document that AI improves performance for in-distribution
tasks but reduces performance for out-of-distribution tasks, particularly for
workers who rely heavily on AI. Equation~\eqref{eq:yS} formalises this: at high
$a_{it}$, the agent's unaided capacity $k_{it}$ is low (from learning atrophy), and
the AI's effective quality $q s(z)$ is low. Both forces simultaneously reduce $y^S$.
\end{remark}

\subsection{Cognitive Capital and Debt Dynamics}
\label{sec:dynamics}

The state variables evolve according to:
\begin{align}
k_{i,t+1} &= (1-\delta)\, k_{it} + \ell(1 - a_{it}) + \nu\, x_{it}, \label{eq:k_law}\\
b_{i,t+1} &= (1 + r_b)\, b_{it} + d(a_{it}) - \rho\, x_{it}, \label{eq:b_law}
\end{align}
where $\delta \in (0,1)$ is the cognitive depreciation rate, $r_b > 0$ is the
cognitive debt compounding rate, $\nu > 0$ is the learning efficiency of
deliberate practice, and $\rho > 0$ is its debt-reduction efficiency.

\begin{assumption}[Dynamics]
\label{ass:dynamics}
\begin{enumerate}
\item[\emph{(i)}] $\ell : [0,1] \to \R_+$ is $C^2$ with $\ell(0) = 0$,
$\ell'(s) > 0$, $\ell''(s) \leq 0$. (Learning by doing is increasing and weakly
concave in unaided effort fraction $1-a$.)
\item[\emph{(ii)}] $d : [0,1] \to \R_+$ is $C^2$ with $d(0) = 0$,
$d'(a) > 0$, $d''(a) > 0$. (Cognitive debt accumulation is increasing and strictly
convex in AI substitution intensity.)
\item[\emph{(iii)}] $r_b > 0$. (Cognitive debt compounds: foundational gaps raise
the marginal cost of subsequent learning and error-detection.)
\end{enumerate}
\end{assumption}

The law of motion for $k$ reflects two learning channels. The term
$\ell(1-a_{it})$ is \emph{learning by doing} from unaided cognition: working
through a problem independently---even imperfectly---builds transferable capacity.
When $a_{it} = 1$, this channel is shut: the agent never engages the reasoning
process. The term $\nu x_{it}$ is \emph{deliberate practice}: targeted effort to
rebuild or maintain cognitive capacity.

The law of motion for $b$ reflects three debt dynamics. The term $(1+r_b) b_{it}$
captures compounding: existing gaps make new gaps easier to acquire (habit formation,
rising prompt dependency, declining verification skills). The term $d(a_{it})$
reflects new debt issuance, convex in $a_{it}$ because marginal substitution increasingly
replaces higher-order reasoning. The term $-\rho x_{it}$ is debt repayment.

We impose the non-negativity constraint $b_{i,t+1} \geq 0$: cognitive debt cannot
become negative (agents cannot be ``ahead'' of their task demands). This is enforced
via the upper bound on deliberate practice:
\begin{equation}
x_{it} \leq \frac{(1+r_b)\,b_{it} + d(a_{it})}{\rho}.
\label{eq:b_nonneg}
\end{equation}
When this constraint is slack, the dynamics follow \eqref{eq:b_law} as written.

\begin{assumption}[Debt compounding dominates depreciation]
\label{ass:compounding}
$r_b > \delta$ and there exists $\bar{a} < 1$ such that for $a \geq \bar{a}$,
$\ell(1-a)/K_t < r_b \cdot \Omega_t$ (the cognitive capital accumulation from
practice does not offset leverage compounding at high substitution rates).
\end{assumption}

Assumption~\ref{ass:compounding} ensures that cognitive debt compounds faster than
cognitive capital depreciates naturally. Without it, tranquil-period leverage could
self-correct through passive depreciation; with it, the leverage ratio drifts upward
unless actively repaid via deliberate practice.

\begin{remark}[Hedge, Speculative, and Ponzi Cognition]
\label{rem:minsky-taxonomy}
Following \citet{Minsky1986}'s financial classification, define:
\begin{itemize}
\item \emph{Hedge cognition}: $k_{it} \geq b_{it} + \text{task demand}$.
The agent can service all cognitive obligations from own capital.
\item \emph{Speculative cognition}: $k_{it} < \text{task demand}$, but
$k_{it} + \text{AI support} \geq \text{task demand}$. The agent meets current
demands only by rolling over AI reliance---paying ``interest'' but not principal.
\item \emph{Ponzi cognition}: $d(a_{it}) + r_b b_{it} > \rho x_{it}^{\max}$.
New debt and compounding exceed maximum repayment capacity.
The agent sustains output only by escalating $a_{it}$, constituting a
cognitive Ponzi scheme.
\end{itemize}
The central theoretical result (Proposition~\ref{prop:minsky}) is that the
decentralised equilibrium systematically migrates agents from hedge through speculative
to Ponzi cognition during tranquil periods.
\end{remark}

\subsection{AI Quality and Model Collapse}
\label{sec:ai_quality}

AI quality $q_t$ is determined at the aggregate level. When agents extensively use
AI-generated content, two degradation forces operate.

\begin{assumption}[AI quality dynamics]
\label{ass:quality}
\begin{equation}
q_{t+1} = \bar{q} \cdot (1 - \gamma \bA_t) \cdot (1 + \eta\, \Omega_t)^{-1},
\label{eq:quality}
\end{equation}
where $\bA_t = \int_0^1 a_{it}\, \diff i$, $\Omega_t = \bB_t / \bK_t$ is the
aggregate cognitive leverage ratio, $\gamma \in (0,1)$ (ensuring $q_{t+1} > 0$
whenever $\bA_t \in [0,1]$), and $\eta > 0$.
\end{assumption}

The first factor $(1 - \gamma \bA_t)$ captures \emph{model collapse}
\citep{ShumailovEtAl2024}: when AI-generated content enters training pipelines at
scale, model quality degrades because the diversity and grounding of the training
distribution shrinks. The second factor $(1 + \eta\Omega_t)^{-1}$ captures
\emph{verification failure}: as aggregate leverage $\Omega_t$ rises, fewer agents
retain the unaided capacity to detect and correct AI errors, so bugs, hallucinations,
and systematic biases propagate unfiltered into knowledge stocks.

\subsection{Information Structure and Belief Dynamics}
\label{sec:beliefs}

Agents are rational conditional on a misspecified perceived law of motion in which
stress arrivals are estimated from realised crisis frequencies. They observe their
own output and whether a crisis occurs, but do not observe the aggregate leverage
ratio $\Omega_t$ or the true crisis probability $\pi_t$. The subjective probability
of the stress state is updated by:
\begin{equation}
\hpi_{t+1} = (1-\lambda)\, \hpi_t + \lambda \cdot \mathbf{1}\{\text{crisis at } t\},
\label{eq:belief}
\end{equation}
where $\lambda \in (0,1)$ is the learning rate. In a period without a crisis,
$\hpi_{t+1} < \hpi_t$: subjective risk falls. This rule is adaptive rather than
fully Bayesian: a true Bayesian who knew $\pi_t = \Pi(\Omega_t, M_t)$ and the
aggregate law of motion would filter $\Omega_t$ from the absence of crises, and
would not necessarily lower $\hat\pi_t$ during an upswing. The updating rule
\eqref{eq:belief} captures the empirically relevant case of misspecified learning
from realised crisis frequencies.

The true crisis probability depends on aggregate leverage and AI market concentration:
\begin{equation}
\pi_t = \Pi(\Omega_t, M_t), \quad \Pi_\Omega > 0,\; \Pi_{\Omega\Omega} \geq 0,\;
\Pi_M > 0,\; \Pi_{\Omega M} \geq 0,\; \Pi(0,\cdot) = 0,\;
\lim_{\Omega \to \infty}\Pi(\Omega,M_t)=1,
\label{eq:true_pi}
\end{equation}
where $M_t$ is the Herfindahl--Hirschman index of AI system market share.
Higher concentration raises crisis probability ($\Pi_M > 0$) because correlated
failures are more likely when agents share the same AI systems, and concentrates
the systemic impact of any single model collapse.
We adopt the parametric form $\Pi(\Omega,M) = 1 - e^{-\lambda_\pi M\,\Omega^\gamma}$
with $\gamma \geq 1$ for the proofs that require explicit computation; the qualitative
results hold for any $\Pi$ satisfying \eqref{eq:true_pi}.

\section{Individual Optimization}
\label{sec:individual}

\subsection{The Agent's Problem}

Agent $i$ maximises discounted expected utility, taking as given the AI quality path
$\{q_t\}$, aggregate variables $\{\Omega_t, \bar{y}_t\}$, and the subjective crisis
probability $\hpi_t$:
\begin{equation}
V(k, b;\, \hpi, q) = \max_{a \in [0,1],\; x \geq 0} \left\{
\tilde{U}(k,a,b;\,\hpi,q,\bar{y})
- c(x) + \beta\, \E\bigl[V(k', b';\, \hpi', q')\bigr]
\right\},
\label{eq:bellman}
\end{equation}
where $c(x)$ is the direct cost of deliberate practice with $c(0) = 0$, $c'(x) > 0$,
$c''(x) > 0$; $\beta \in (0,1)$ is the discount factor; $k', b'$ follow
\eqref{eq:k_law}--\eqref{eq:b_law}; and
\[
\tilde{U} \equiv (1-\hpi)\,u^N(k,a;q,\bar{y}) + \hpi\,\E_z\bigl[u^S(k,a,b;q,z,\bar{y})\bigr],
\qquad u^s = y^s - \chi\max\bigl\{0,\,\bar{y} - y^s\bigr\}.
\]
In the symmetric equilibrium $y^N_{it} = \bar{y}_t$, so the benchmark penalty is
zero in realisation. The individual FOCs derived below treat $\bar{y}_t$ as fixed,
yielding the same conditions as if $\chi=0$. The social planner, however, internalises
$\partial\bar{y}_t/\partial A_t = K_t G_a(A_t;q_t) > 0$, generating the arms-race
externality quantified in Proposition~\ref{prop:welfare}.

To characterise individual incentives in closed form and derive comparative statics,
we work with a two-period version of the problem. The infinite-horizon problem
yields qualitatively identical results (Appendix~\ref{app:infinite}).

\subsubsection*{Two-period analytical model}

An agent lives for two periods, $t=0$ and $t=1$. In $t=0$ she chooses $a$ and $x$.
In $t=1$ she consumes the residual value of her state. The period-1 value function is:
\begin{equation}
V_1(k', b') = \mu_k\, k' - \mu_b\, b',
\label{eq:V1}
\end{equation}
where $\mu_k > 0$ and $\mu_b > 0$ are shadow values of cognitive capital and debt
in the terminal period (derived from the infinite-horizon envelope conditions in
Appendix~\ref{app:shadow}). The objective is:
\begin{equation}
\max_{a, x} \;\underbrace{(1-\hpi)\, k G(a;q) + \hpi\, \E_z\bigl[k\tilde{G}(a;q,z) - \kappa z b\bigr]}_{\text{expected current output}}
- c(x) + \beta\, V_1(k', b').
\label{eq:2period}
\end{equation}

Expanding using \eqref{eq:k_law}--\eqref{eq:b_law} and \eqref{eq:V1}:
\begin{equation}
\max_{a,x} \;\bigl[(1-\hpi)\, k G(a;q) + \hpi\,\E_z[k\tilde{G}(a;q,z)]\bigr]
- \hpi\, \kappa\, \bar{z}\, b
- c(x)
+ \beta\,\mu_k\,\bigl[(1-\delta)k + \ell(1-a) + \nu x\bigr]
- \beta\,\mu_b\,\bigl[(1+r_b)b + d(a) - \rho x\bigr],
\label{eq:2period_expanded}
\end{equation}
where $\bar{z} = \E[z \mid S]$.

\subsubsection*{First-order conditions}

Differentiating \eqref{eq:2period_expanded} with respect to $a$ yields the
first-order condition for AI substitution:
\begin{equation}
\boxed{
k\,\bigl[(1-\hpi)\,G_a(a^*;q) + \hpi\,\tilde{G}_a(a^*;q,\bar{z})\bigr]
= \beta\,\bigl[\mu_k\,\ell'(1-a^*) + \mu_b\,d'(a^*)\bigr].
}
\label{eq:foc_a}
\end{equation}
The left-hand side is the expected marginal product of AI in current output.
The right-hand side is the discounted marginal future cost: forgone learning by doing
(valued at $\mu_k$) plus newly issued cognitive debt (valued at $\mu_b$).

Differentiating with respect to $x$ yields the condition for deliberate practice:
\begin{equation}
c'(x^*) = \beta\,(\mu_k\,\nu + \mu_b\,\rho).
\label{eq:foc_x}
\end{equation}
Practice is chosen until its marginal cost equals the discounted return: increasing
cognitive capital (at rate $\nu$, valued at $\mu_k$) plus reducing cognitive debt
(at rate $\rho$, valued at $\mu_b$).

\subsection{Proposition 1: Rational Cognitive Debt}

\begin{proposition}[Rational Cognitive Debt]
\label{prop:individual}
Under Assumptions~\ref{ass:G}--\ref{ass:dynamics}, for any $\hpi \in [0,1)$,
$q > 0$, and $k > 0$, suppose additionally that
\begin{equation}
k\,\bigl[(1-\hpi)\,G_a(1;\,q) + \hpi\,\tilde{G}_a(1;\,q,\bar{z})\bigr]
< \beta\,\bigl[\mu_k\,\ell'(0) + \mu_b\,d'(1)\bigr].
\label{eq:boundary1}
\end{equation}
(The marginal future cost of full substitution exceeds its marginal current benefit
at $a=1$.) Then:
\begin{enumerate}
\item[\emph{(i)}] There exists a unique interior solution $a^* \in (0,1)$ to
\eqref{eq:foc_a}, and the resulting cognitive debt issuance satisfies $d(a^*) > 0$.
\item[\emph{(ii)}] The equilibrium AI substitution intensity satisfies:
\[
\frac{\partial a^*}{\partial q} > 0, \quad
\frac{\partial a^*}{\partial k} > 0, \quad
\frac{\partial a^*}{\partial \hpi} < 0, \quad
\frac{\partial a^*}{\partial \beta} < 0, \quad
\frac{\partial a^*}{\partial \mu_b} < 0, \quad
\frac{\partial a^*}{\partial \kappa} < 0.
\]
\item[\emph{(iii)}] The ``cognitive debt wedge'' satisfies:
\[
a^*(\hpi) < a^{\text{no-debt}}
\]
where $a^{\text{no-debt}}$ is the optimal AI use when $d \equiv 0$, confirming that
debt accumulation is the mechanism restraining (but not eliminating) AI substitution.
\end{enumerate}
\end{proposition}

\begin{proof}
See Appendix~\ref{app:prop1}.
\end{proof}

\begin{remark}[Economic interpretation of comparative statics]
Part (ii) translates directly into policy-relevant predictions:
\begin{itemize}
\item $\partial a^*/\partial q > 0$: AI capability improvements raise equilibrium
debt issuance---a form of ``keeping up'' with the technology that progressively
erodes the cognitive base.
\item $\partial a^*/\partial k > 0$: Agents with higher cognitive capital use AI
more intensively (the multiplicative return is higher), generating higher gross debt.
This creates a non-monotone inequality trajectory: high-$k$ agents gain more
short-run productivity but face larger potential capital erosion.
\item $\partial a^*/\partial \hpi < 0$: Higher perceived stress probability deters
AI substitution. This is the ``insurance'' channel: agents who believe AI may fail
invest in maintaining their unaided capacity.
\item $\partial a^*/\partial \beta < 0$: More patient agents internalise more of the
future cost of cognitive capital erosion. Short-termism is a direct driver of
excessive AI reliance.
\end{itemize}
\end{remark}

\section{Aggregate Dynamics and the Cognitive Minsky Moment}
\label{sec:aggregate}

\subsection{Competitive Equilibrium}

\begin{definition}[Competitive Equilibrium]
A competitive equilibrium is a sequence of individual policies $\{a_{it}^*, x_{it}^*\}$,
aggregate variables $\{K_t, B_t, A_t, \Omega_t\}$, AI quality $\{q_t\}$, and belief
sequences $\{\hpi_t, \pi_t\}$ such that:
\begin{enumerate}
\item[\emph{(i)}] Each agent solves \eqref{eq:bellman} taking $\{q_t, \hpi_t\}$
as given.
\item[\emph{(ii)}] Markets clear: $K_t = \int k_{it}\diff i$,
$B_t = \int b_{it}\diff i$, $A_t = \int a_{it}\diff i$,
$\Omega_t = B_t/K_t$.
\item[\emph{(iii)}] AI quality follows \eqref{eq:quality}.
\item[\emph{(iv)}] Beliefs follow \eqref{eq:belief} and \eqref{eq:true_pi}.
\end{enumerate}
\end{definition}

We focus on symmetric equilibria in which $a_{it}^* = a_t^*$ and $x_{it}^* = x_t^*$
for all $i$, so $K_t = k_t$ and $B_t = b_t$ in the representative-agent
characterisation. Heterogeneous-agent extensions are developed in
Section~\ref{sec:hetero}.

\subsection{Proposition 2: The Cognitive Minsky Moment}

\begin{proposition}[Cognitive Minsky Moment]
\label{prop:minsky}
Under Assumptions~\ref{ass:G}--\ref{ass:compounding} and adaptive (misspecified)
belief updating \eqref{eq:belief}, suppose a tranquil period
$\mathcal{T} = \{0, 1, \ldots, T-1\}$ in which no crisis is realised, and that the
belief-updating effect dominates the endogenous AI-quality degradation effect:
\begin{equation}
\left|\frac{\partial a^*}{\partial \hpi}\,\Delta\hpi_t\right|
> \left|\frac{\partial a^*}{\partial q}\,\Delta q_t\right|
\quad \text{for all } t \in \mathcal{T}.
\label{eq:dominance}
\end{equation}
Then along any equilibrium path through $\mathcal{T}$:
\begin{enumerate}
\item[\emph{(i)}] Subjective risk is strictly decreasing: $\hpi_{t+1} < \hpi_t$
for all $t \in \mathcal{T}$.
\item[\emph{(ii)}] Equilibrium AI substitution intensity is strictly increasing:
$a_{t+1}^* > a_t^*$ for all $t \in \mathcal{T}$.
\item[\emph{(iii)}] Aggregate cognitive leverage is strictly increasing:
$\Omega_{t+1} > \Omega_t$ for all $t \in \mathcal{T}$.
\item[\emph{(iv)}] True crisis probability is strictly increasing:
$\pi_{t+1} = \Pi(\Omega_{t+1},M_t) > \Pi(\Omega_t,M_t) = \pi_t$ for all $t \in \mathcal{T}$.
\item[\emph{(v)}] There exists $T^* < \infty$ such that for all $t > T^*$,
$\hpi_t < \pi_t$ \emph{(the Minsky divergence)}: perceived risk falls below true
systemic risk.
\end{enumerate}
\end{proposition}

\begin{proof}
See Appendix~\ref{app:prop2}.
\end{proof}

\begin{remark}[Dominance condition and learning specification]
Condition \eqref{eq:dominance} requires that belief-driven increases in AI
substitution outweigh the dampening effect from endogenously falling AI quality
($\partial a^*/\partial q > 0$ while $q_t$ falls from \eqref{eq:quality}).
Without this condition, AI quality erosion would attenuate or reverse the Minsky
dynamics---an important general-equilibrium feedback. Condition \eqref{eq:dominance}
is satisfied when $\gamma \eta$ (the quality-degradation sensitivity) is small
relative to the belief-learning rate $\lambda$, or when the tranquil period is short
enough that $\Delta q_t$ is negligible. Identifying this dominance empirically---via
exogenous variation in AI quality and belief-updating parameters---is a key direction
for future work.

The divergence $\hpi_t < \pi_t$ (part (v)) arises because agents update on
realised crises but cannot infer the true $\pi_t = \Pi(\Omega_t, M_t)$ from
crisis-free periods. Under misspecified learning, each crisis-free period pushes
$\hat\pi_t$ down regardless of the growing leverage. A correctly specified Bayesian
who filtered $\Omega_t$ from the data would not necessarily lower $\hat\pi_t$
during an upswing; the divergence is therefore a consequence of the misspecification,
not of the underlying endogeneity.
\end{remark}

\begin{figure}[t]
\centering
\includegraphics[width=\textwidth]{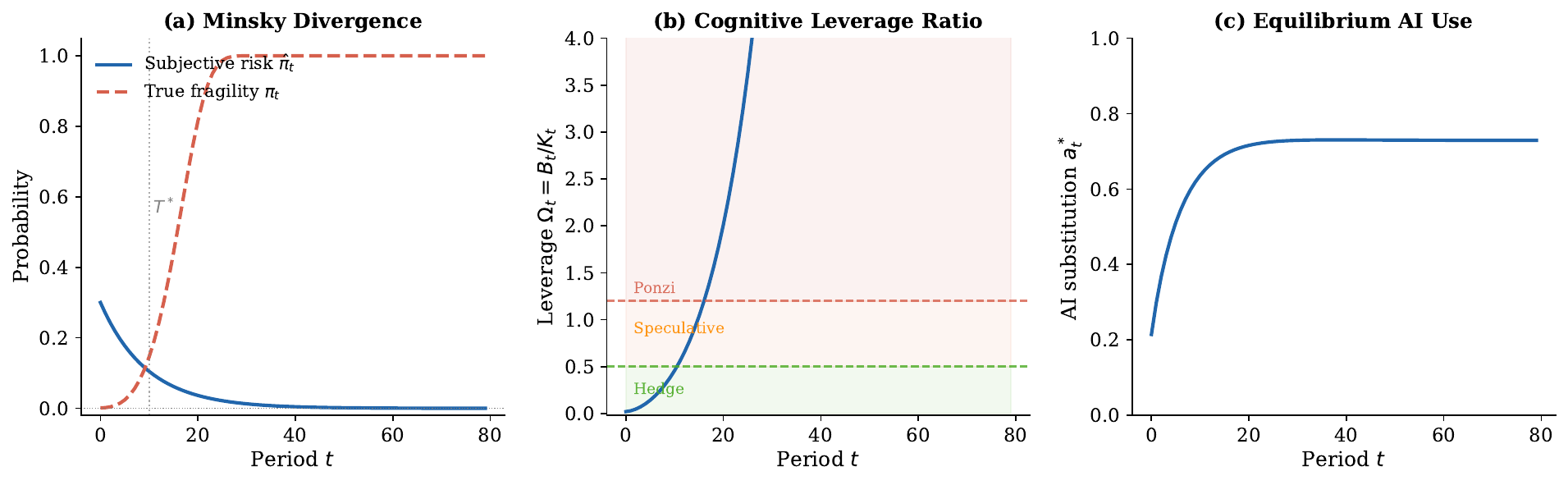}
\caption{\textbf{Minsky Divergence and Cognitive Leverage Dynamics.}
Simulated equilibrium path under the parametric model
($q=1$, $\phi=0.35$, $\eta=0.60$, $\delta=0.04$, $r_b=0.10$, $\beta=0.95$,
$\hat\pi_0=0.30$, no crises realised for $T=80$ periods).
\textit{Panel (a):} Subjective risk $\hat\pi_t$ falls monotonically while true
fragility $\pi_t$ rises; the vertical dotted line marks the crossing point $T^*$.
\textit{Panel (b):} The aggregate cognitive leverage ratio $\Omega_t = B_t/K_t$
drifts upward through the Hedge, Speculative, and Ponzi zones (defined at
$\Omega=0.50$ and $\Omega=1.20$ respectively).
\textit{Panel (c):} Equilibrium AI substitution intensity $a_t^*$ rises as
subjective risk falls, closing the feedback loop.}
\label{fig:minsky}
\end{figure}

\subsection{Proposition 3: Convexity of Crisis Losses}

Define the expected crisis loss at date $t$ as:
\begin{equation}
L_t = \Pi(\Omega_t, M_t) \cdot \kappa\, \E_z\bigl[\max\{0,\, z\, B_t - \mathcal{V}(K_t)\}\bigr],
\label{eq:loss}
\end{equation}
where $\mathcal{V}(K) = K^\alpha$ ($\alpha \in (0,1)$) is the aggregate verification
capacity---the concave function representing society's ability to detect and correct AI
errors as a function of aggregate cognitive capital.

\begin{proposition}[Convex Fragility]
\label{prop:convex}
Under Assumptions~\ref{ass:G}--\ref{ass:quality} and the crisis probability
specification \eqref{eq:true_pi}, the expected crisis loss $L_t$ is:
\begin{enumerate}
\item[\emph{(i)}] Strictly increasing in the aggregate leverage ratio: $\partial L_t / \partial \Omega_t > 0$.
\item[\emph{(ii)}] Locally convex near the fragility threshold: $\partial^2 L_t / \partial \Omega_t^2 > 0$
in a neighbourhood of the fragility threshold $\bar{\Omega}$.
\item[\emph{(iii)}] Strictly increasing in AI model concentration: $\partial L_t / \partial M_t > 0$,
with complementarity $\partial^2 L_t / \partial \Omega_t \partial M_t \geq 0$
(concentration amplifies the marginal cost of leverage).
\end{enumerate}
\end{proposition}

\begin{proof}
See Appendix~\ref{app:prop3}.
\end{proof}

\begin{remark}[Why convexity matters]
Convexity of $L_t$ in $\Omega_t$ implies that small additional increments of
cognitive leverage, once leverage is already high, produce disproportionately
large expected losses. This gives a formal rationale for the precautionary principle:
a society near the fragility threshold should weight leverage reduction much more
than one near the safe zone, even if their current expected losses are similar.
\end{remark}

\begin{figure}[t]
\centering
\includegraphics[width=0.62\textwidth]{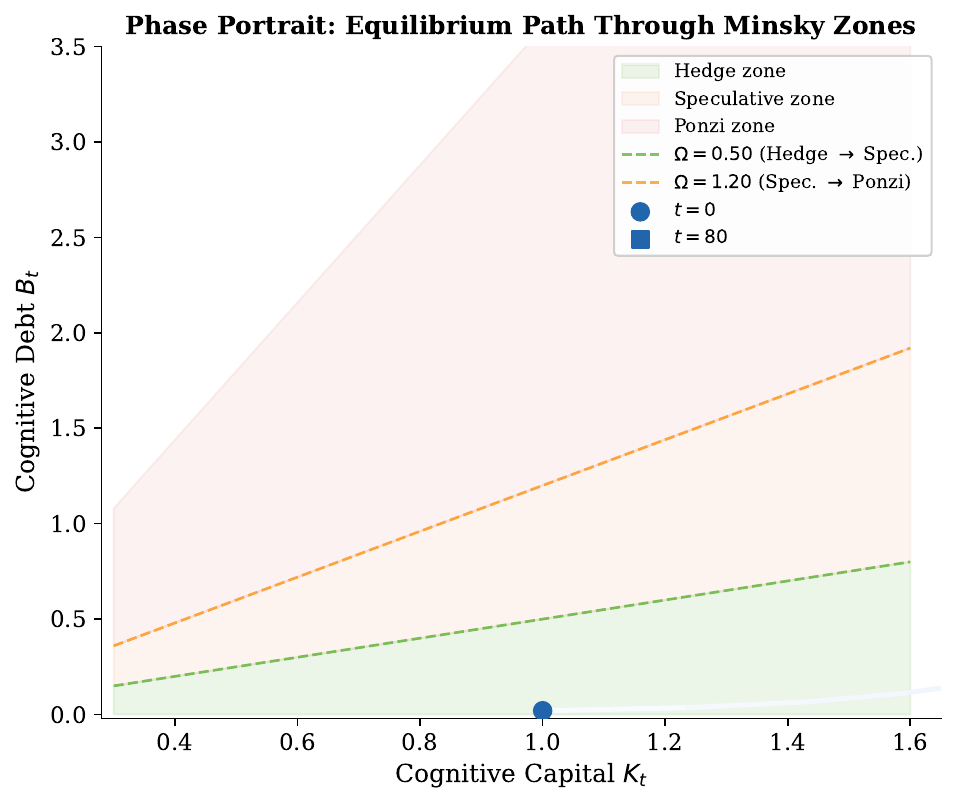}
\caption{\textbf{Phase Portrait in Cognitive Capital--Debt Space.}
The equilibrium trajectory (blue gradient line, darker = later) starts in the
Hedge zone ($\Omega < 0.50$) and migrates through the Speculative zone into
the Ponzi zone during an uninterrupted tranquil period.
The colour intensity encodes time: early periods are light blue, late periods
are dark blue. Arrows indicate the direction of motion.
Iso-leverage lines at $\Omega = 0.50$ (green dashed) and $\Omega = 1.20$
(orange dashed) delimit the three Minsky zones.}
\label{fig:phase}
\end{figure}

\section{Post-Crisis Dynamics: The False-Correction Loop}
\label{sec:fcl}

\subsection{Setup}

Suppose a stress event occurs at date $t=0$: the stress state is realised with
intensity $z_0 > 0$. Agents observe the shortfall in their unaided capacity.
Do they reduce AI substitution intensity---repaying cognitive debt---or do they
increase it to maintain output?

Let $w > 0$ denote the competitive output target (the market-determined
minimum output level required for continued participation in the labour market).
Agents face a binding output constraint:
\begin{equation}
y_{i0} \geq w. \label{eq:output_constraint}
\end{equation}
This constraint captures the competitive pressure documented in the literature:
AI adoption by some agents shifts the performance distribution, raising the
benchmark against which all agents are evaluated \citep{BrynjolfssonEtAl2025}.

\begin{definition}[False-Correction Loop]
A \emph{false-correction loop} occurs when, following a crisis at $t=0$,
the equilibrium AI substitution intensity satisfies $a_{1}^* > a_{0}^-$,
where $a_0^-$ denotes the pre-crisis substitution intensity, and when the
resulting leverage $\Omega_1 > \Omega_{0}^-$.
\end{definition}

\begin{proposition}[False-Correction Loop]
\label{prop:fcl}
Suppose the output constraint \eqref{eq:output_constraint} is binding at $t=1$,
that $k_{1} < k_{0}^-$ (cognitive capital has declined through the crisis period),
and that the post-crisis deliberate-practice response is insufficient to offset
additional debt issuance: $d(a_1^*) - \rho\, x_1^* > 0$.
Then the post-crisis equilibrium AI substitution intensity satisfies
$a_1^* > a_0^-$ if and only if:
\begin{equation}
\underbrace{\lambda_y\, k_1\, G_a(a_0^-;\, q)}_{\text{shadow value of current output}}
> \underbrace{\beta\,\bigl[\mu_k\,\ell'(1-a_0^-) + \mu_b\,d'(a_0^-)\bigr]}_{\text{shadow value of cognitive capital recovery}}.
\label{eq:fcl_condition}
\end{equation}
When conditions \eqref{eq:fcl_condition} and $d(a_1^*) - \rho x_1^* > 0$ both
hold, the post-crisis path satisfies $\Omega_{t+1} > \Omega_t$ for all $t \geq 1$
until a new crisis occurs, with successive crises of non-decreasing severity.
\end{proposition}

\begin{proof}
See Appendix~\ref{app:prop4}.
\end{proof}

\begin{remark}[When does the false-correction loop operate?]
Condition \eqref{eq:fcl_condition} is more likely to be satisfied when: (a) the
output target $w$ is high (high $\lambda_y$); (b) cognitive capital $k_1$ is low
following the crisis (paradoxically making the condition harder to satisfy in the
left-hand side, but also reducing the option value of capital recovery); (c) the
discount factor $\beta$ is low (short-termism); (d) the debt compounding rate $r_b$
is high (making capital recovery expensive). In practice, the condition is most
likely binding in competitive labour markets where output is continuously monitored
and the opportunity cost of a period of reduced output is high.
\end{remark}

\begin{remark}[Financial analogy]
The false-correction loop is the cognitive analogue of the ``leveraged buyout of a
failing firm'': when the firm cannot service its existing debt, the rational response
in a competitive product market is to borrow more to maintain operations, rolling
over obligations until eventual insolvency. Here, the firm is the agent's cognitive
system, the debt is cognitive debt, and the operations are current output production.
The loop is individually rational but socially destructive.
\end{remark}

\section{Welfare Analysis}
\label{sec:welfare}

\subsection{Three Externalities}

The decentralised equilibrium exhibits three distinct externalities:

\paragraph{1. Systemic risk externality.} Decentralised agents take $\Omega_t$ as
given, and therefore omit the marginal systemic cost
$\Pi_\Omega(\Omega_t,M_t)\cdot(\partial\Omega_t/\partial A_t)\cdot L_t^{\text{cond}}$
that the planner internalises, where $L_t^{\text{cond}} = L_t/\Pi(\Omega_t,M_t)$ is
the expected conditional loss.

\paragraph{2. Cognitive public-goods externality.} Society's verification capacity
$\mathcal{V}(K_t)$ is a public good: all agents benefit from a high-$K_t$ population
that can scrutinise AI outputs, train the next generation, and provide error-correction
in shared knowledge domains. When agent $i$ substitutes AI for cognition, $K_t$
falls, degrading this public good. The individual internalises the private future
value of her own $k$ (at shadow price $\mu_k^D$), but not the aggregate spillover;
the social planner values cognitive capital at $\mu_k^P > \mu_k^D$, reflecting the
public-good component. The externality is $(\mu_k^P - \mu_k^D)$ per unit of
forgone practice.

\paragraph{3. Arms-race externality.} The competitive benchmark $\bar{y}_t$ is
endogenous to aggregate $A_t$: when the planner raises $A_t$, the benchmark
$\bar{y}_t = K_t G(A_t;q_t)$ rises, imposing a cost on all agents through the
utility function \eqref{eq:benchmark}. The individual, taking $\bar{y}_t$ as
fixed, ignores $\partial\bar{y}_t/\partial A_t = K_t G_a(A_t;q_t) > 0$.
This strategic complementarity drives aggregate AI adoption above the social
optimum: all agents would prefer the lower-$A$ equilibrium, but individually
each has an incentive to maintain high $a$ to avoid the benchmark penalty.

\subsection{The Social Planner's Problem}

A benevolent social planner chooses aggregate $\{A_t, X_t\}$ to maximise the sum
of agents' welfare, internalising all externalities:
\begin{equation}
\max_{\{A_t, X_t\}} \sum_{t=0}^{\infty} \beta^t \left\{
(1-\pi_t)\, K_t\, G(A_t;\, q_t) + \pi_t\,\E_z\bigl[K_t\,\tilde{G}(A_t;\, q_t, z) - \kappa z B_t\bigr] - C(X_t)
\right\},
\label{eq:planner}
\end{equation}
subject to \eqref{eq:k_law}--\eqref{eq:b_law} in aggregate, \eqref{eq:quality},
and \eqref{eq:true_pi}, taking $\pi_t = \Pi(\Omega_t, M_t)$ as endogenous.

\begin{proposition}[Decentralised Over-adoption]
\label{prop:welfare}
Under Assumptions~\ref{ass:G}--\ref{ass:quality}, the decentralised equilibrium
aggregate AI substitution intensity exceeds the social optimum:
\[
A_t^D > A_t^P \quad \text{for all } t \geq 0.
\]
The gap $\Delta_t \equiv A_t^D - A_t^P$ is strictly increasing in: (i) aggregate
model concentration $M_t$; (ii) the length of the current tranquil period; and
(iii) the convexity parameter of the debt accumulation function $d$.

The constrained-optimal Pigouvian tax on AI substitution intensity is:
\begin{equation}
\tau_t^* =
\underbrace{\beta\,\Pi_\Omega(\Omega_t,M_t)\,\frac{\partial \Omega_t}{\partial A_t}\,
L_t^{\text{cond}}}_{\text{systemic risk externality}}
+\underbrace{\beta\,(\mu_k^P - \mu_k^D)\,\ell'(1-A_t^P)}_{\text{cognitive public-good externality}}
+\underbrace{\chi\,K_t\,G_a(A_t^P;\,q_t)}_{\text{arms-race externality}},
\label{eq:tax}
\end{equation}
evaluated at the social optimum $(A_t^P, X_t^P)$. The first term captures the
marginal increase in expected crisis losses from raising aggregate leverage; it is
positive since $\Pi_\Omega > 0$. The second term uses the shadow-value gap
$\mu_k^P - \mu_k^D > 0$ to avoid double-counting the privately internalised return
to cognitive capital; it is positive since both the gap and $\ell'>0$. The third
term is positive since $\chi \geq 0$ and $G_a > 0$. The optimal tax is (a)
increasing in aggregate leverage $\Omega_t$, (b) increasing in model concentration
$M_t$, and (c) counter-cyclical---rising during tranquil periods as leverage
accumulates.
\end{proposition}

\begin{proof}
See Appendix~\ref{app:prop5}.
\end{proof}

\begin{remark}[Policy instruments]
Equation~\eqref{eq:tax} suggests several implementable policies. A levy indexed to
the aggregate AI usage share (approximating $\Omega_t$) serves as the systemic risk
tax. Mandatory unaided assessment requirements---professional certification exams
taken without AI, stress-test protocols for AI-dependent workflows---serve as
forced debt repayment ($x_t$). Limits on AI model market concentration address $M_t$.
The counter-cyclicality result provides a formal rationale for tightening these
requirements during periods of rapid AI adoption, analogous to counter-cyclical
capital buffers in macroprudential regulation.
\end{remark}

\section{Heterogeneous Agents and the Reversal of Fortune}
\label{sec:hetero}

\subsection{Setup}

We enrich the model with two types of agents: type $H$ (high initial cognitive capital,
$k_{H,0} > \bar k$) and type $L$ (low initial cognitive capital, $k_{L,0} < \bar k$),
in equal proportions. Both types begin with zero cognitive debt ($b_{H,0} = b_{L,0} = 0$)
and the same subjective risk belief $\hat\pi_0$. All structural parameters are identical
across types.

From Proposition~\ref{prop:individual}, part (ii), $\partial a^*/\partial k > 0$:
higher cognitive capital raises the return to AI (via the multiplicative structure
$\partial y^N/\partial a = k \cdot G_a$), so the high-$k$ type adopts AI more
intensively in equilibrium:
\begin{equation}
a_{H,t}^* > a_{L,t}^* \quad \text{for all } t \text{ such that } k_{H,t} > k_{L,t}.
\label{eq:hetero_ranking}
\end{equation}

\subsection{Proposition 6: Reversal of Fortune}

Define the \emph{cognitive capital gap} $\Delta k_t \equiv k_{H,t} - k_{L,t}$
and the \emph{AI adoption gap} $\Delta a_t^* \equiv a_{H,t}^* - a_{L,t}^* > 0$.

\begin{proposition}[Reversal of Fortune]
\label{prop:hetero}
In the two-type economy, during a tranquil period with no crises:
\begin{enumerate}
\item[\emph{(i)}] $\Delta k_t$ is strictly decreasing for all $t \geq 0$:
high-$k$ agents erode their cognitive capital advantage faster than low-$k$ agents.
\item[\emph{(ii)}] $\Delta a_t^*$ is initially positive but converges to zero as
$\Delta k_t \to 0$.
\item[\emph{(iii)}] If the learning-by-doing differential satisfies
\begin{equation}
\ell'(\xi)\,\Delta a_t^* > (1-\delta)\,\Delta k_t
\quad \text{for all } t \geq T^\dagger,
\label{eq:reversal_condition}
\end{equation}
for some $T^\dagger < \infty$, then there exists $T^{**} < \infty$ such that
$\Delta k_{T^{**}} = 0$ and $\Delta k_t < 0$ for all $t > T^{**}$:
\emph{the high-$k$ type ends up with less cognitive capital than the low-$k$ type.}
\item[\emph{(iv)}] The aggregate crisis loss in the heterogeneous economy exceeds
that in the homogeneous economy with $\bar k = (k_{H,0} + k_{L,0})/2$:
\[
L_t^{\text{hetero}} > L_t^{\text{homo}},
\]
because the high-$k$ type's elevated leverage generates a disproportionate increase
in expected losses (by convexity of $L$ in $\Omega$, Proposition~\ref{prop:convex}).
\end{enumerate}
\end{proposition}

\begin{proof}
See Appendix~\ref{app:prop6}.
\end{proof}

\begin{remark}[Interpretation]
Part (iii) is the central result. It says that the agents who benefit \emph{most}
from AI adoption in the short run---because their high $k$ makes the multiplicative
return to $a$ large---accumulate cognitive debt fastest, erode their capital most
severely, and may eventually fall below the initial laggards. The mechanism is
entirely within the rational-agent framework: each type is optimising correctly
given private information. The reversal is a general-equilibrium consequence of
the multiplicative structure and the compounding debt dynamics.

Part (iv) implies that inequality is not merely a distributional concern: it is a
source of additional systemic fragility. A more dispersed distribution of cognitive
capital (for given mean) produces higher aggregate crisis losses, providing a
novel rationale for policies that maintain a minimum floor of unaided cognitive
capacity.
\end{remark}

\begin{figure}[t]
\centering
\includegraphics[width=\textwidth]{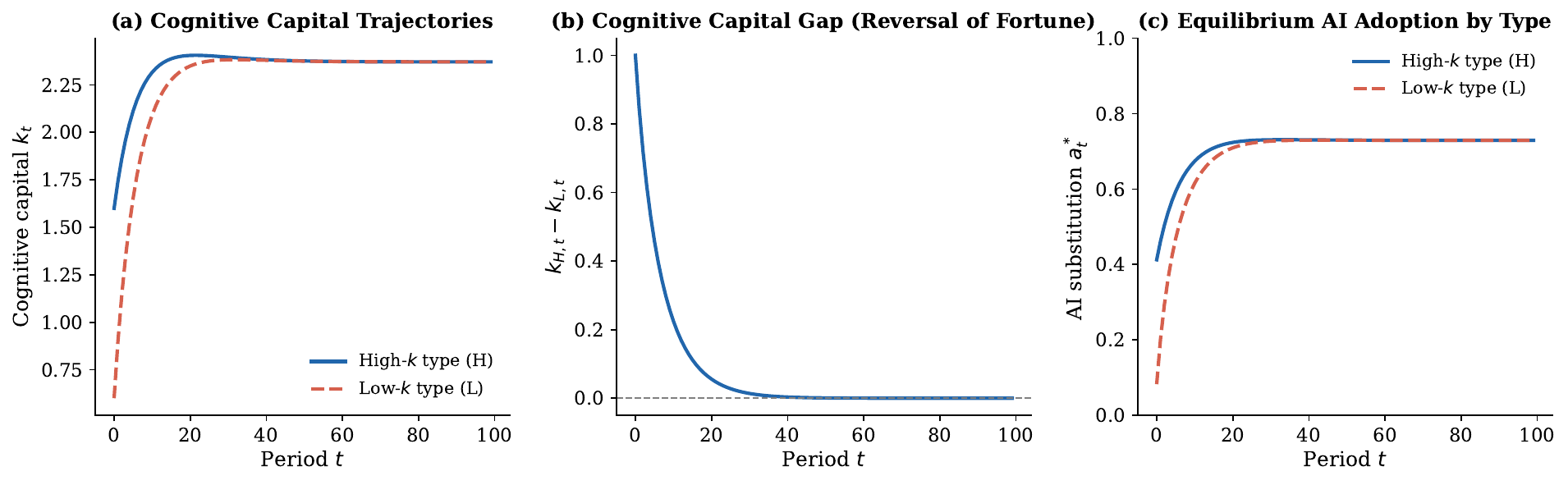}
\caption{\textbf{Cognitive Capital Inequality and the Reversal of Fortune.}
Simulated two-type economy ($k_{H,0}=1.6$, $k_{L,0}=0.6$; all other parameters
as in Figure~\ref{fig:minsky}).
\textit{Panel (a):} Cognitive capital trajectories: the high-$k$ type (blue solid)
declines faster and eventually converges toward---and crosses---the low-$k$ type
(red dashed).
\textit{Panel (b):} The cognitive capital gap $k_{H,t}-k_{L,t}$ falls from its
initial positive value, crosses zero at $T^{**}$ (vertical dotted line), and
becomes negative---the reversal of fortune.
\textit{Panel (c):} AI adoption intensity: the high-$k$ type always adopts AI
more intensively (Proposition~\ref{prop:individual}, $\partial a^*/\partial k>0$),
which is the direct driver of the faster capital erosion.}
\label{fig:hetero}
\end{figure}

\section{Discussion}
\label{sec:discussion}

\subsection{Heterogeneity and Inequality}

Proposition~\ref{prop:hetero} establishes that the reversal of fortune is not a
knife-edge result: it holds whenever condition~\eqref{eq:reversal_condition} is
satisfied, which is generically true when the learning-by-doing differential is
sufficiently large relative to the remaining capital gap. The parametric simulations
in Figure~\ref{fig:hetero} illustrate a concrete instance. The result has two
policy implications. First, policies that protect the cognitive capital floor
(mandatory unaided-practice requirements, AI-free examinations) disproportionately
benefit high-adopting, high-skill agents who face the strongest erosion pressure.
Second, the distributional reversal implies that standard skill-biased-technological-change
frameworks---which predict widening inequality---may mischaracterise the long-run
distribution of cognitive capital under AI adoption.

\subsection{Robustness to the Complementarity Assumption}
\label{sec:robustness}

The baseline production function $y^N_{it} = k_{it}\,G(a_{it};\,q_t)$ imposes
strong multiplicative complementarity: the marginal product of AI is proportional
to $k_{it}$, ensuring $\partial a^*/\partial k > 0$ (Proposition~\ref{prop:individual})
and underpinning the reversal of fortune (Proposition~\ref{prop:hetero}).

Consider the more general class:
\[
y^N_{it} = F(k_{it},\, a_{it};\, q_t), \qquad F_k > 0,\; F_a > 0,\; F_{ka} \text{ unrestricted},
\]
with the baseline corresponding to $F = k\,G(a;\,q)$ (where $F_{ka} = G_a > 0$).
Empirically, \citet{BrynjolfssonEtAl2025} find that generative AI raises
productivity across skill groups, with heterogeneous effects suggesting
$F_{ka}$ may be positive but finite, or even negative for some tasks.

Under this general $F$:
\begin{itemize}
\item \textbf{Propositions~\ref{prop:individual}--\ref{prop:welfare} survive}
provided that $F$ inherits the Inada and concavity conditions of Assumption~\ref{ass:G}
and the dynamics of Assumption~\ref{ass:dynamics} are preserved. The debt
accumulation, Minsky divergence, convex fragility, false-correction loop, and
welfare wedge are all driven by the compounding dynamics \eqref{eq:b_law} and the
misspecified belief updating \eqref{eq:belief}, not by strong complementarity per se.
\item \textbf{Proposition~\ref{prop:hetero} (reversal of fortune) requires
$F_{ka}$ sufficiently large.} The reversal depends on $\partial a^*/\partial k > 0$
(high-$k$ agents adopt more). If $F_{ka} \leq 0$ (AI benefits low-skill workers
more at the margin), then $\partial a^*/\partial k \leq 0$, the adoption ranking
reverses, and the reversal of fortune does not occur---instead, low-skill agents
erode their capital faster. Proposition~\ref{prop:hetero} should therefore be
read as characterising the case of strong complementarity, which is the empirically
relevant case for complex cognitive tasks \citep{DellAcquaEtAl2023} but may not
hold universally.
\end{itemize}

\subsection{Complementary vs. Substitutive AI}

The model throughout treats $a_{it}$ as substitutive AI use. Complementary AI
use---Socratic tutoring, friction-generating feedback, forced verification---appears
in the model as negative $d(a)$ (debt repayment) and positive $\ell$ (learning
activation). The model therefore nests complementary AI use as a special case with
reversed dynamics: complementary AI raises $k$ and reduces $b$. The key policy
question is how to distinguish substitutive from complementary AI use in practice;
the model provides a formal criterion: an AI interaction is complementary if it
increases the agent's ability to perform the relevant task unaided.

\subsection{The ``Cognitive Reserve'' Policy}

An analogy from financial regulation is the \emph{capital adequacy ratio}: banks
must hold a minimum ratio of capital to risk-weighted assets. The cognitive analogue
is a \emph{cognitive reserve requirement}: a minimum ratio $k_{it} / b_{it} \geq
\underline{\Omega}^{-1}$. When this constraint binds, agents must either reduce debt
(through deliberate practice) or reduce issuance (by reducing $a_{it}$). This
constraint would prevent Ponzi cognition as defined in Remark~\ref{rem:minsky-taxonomy}.
Formal analysis of the transition dynamics under such a constraint is left for
future work.

\subsection{Limitations}

Several simplifications warrant acknowledgement. First, we treat the cognitive capital
state variable as one-dimensional; in reality, cognitive skills are domain-specific
and heterogeneous. Second, the model assumes that the production function $G$ is
known and stationary; in a world where AI capabilities are rapidly evolving, $G$
shifts over time and the appropriate value of $a$ changes accordingly.
Third, we abstract from the possibility that AI may be used to enhance cognitive
capital formation---the ``Uzawa-learning'' channel \citep{Uzawa1965, Lucas1988}.
Incorporating endogenous AI quality into the human capital accumulation equation is
a promising extension. Fourth, the model does not account for the possibility that
some forms of cognitive capital---particularly tacit, embodied, and physical knowledge
\citep{Aschenbrenner2024}---may be non-reclaimable once lost.

\section{Conclusion}
\label{sec:conclusion}

We have developed a formal theory of cognitive debt: the unobserved stock of
reasoning obligations that accumulates when individuals systematically outsource
first-principles cognition to AI. The model shows that this debt arises from
rational optimisation, compounds through habit formation and competitive pressure,
and generates Minsky-style fragility: tranquil periods of AI-driven productivity
systematically build the conditions for cognitive crises.

The six propositions identify the key structural forces. Rational agents borrow
cognitive debt because the costs are deferred and partially externalised. Aggregate
leverage rises during tranquil periods because belief updating and output competition
both push toward more AI use. Crisis losses are convex in leverage, so the tail risk
is large. Post-crisis adjustment follows the false-correction loop rather than genuine
deleveraging. And the decentralised equilibrium over-adopts AI relative to the social
optimum by a gap that grows during boom periods and is larger in more concentrated AI
markets.

The policy implications are concrete: an AI-use tax indexed to aggregate leverage;
mandatory unaided-performance audits; limits on AI model concentration; and
deliberate-practice requirements for professionals in high-stakes domains. These
instruments map directly to the three externalities identified in the welfare analysis.

More broadly, the framework offers a formal vocabulary for a concern that has been
articulated informally but not yet analytically: that the short-run productivity
gains from AI adoption may be purchased against a long-run degradation of the
cognitive capacity required to verify, correct, extend, and ultimately reproduce
those gains. Whether this trade-off is empirically large is an open question. The
theory provides the structure within which to answer it.

\newpage
\bibliographystyle{plainnat}
\bibliography{refs}

\newpage
\appendix
\section*{Appendix: Proofs}

\setcounter{section}{1}
\renewcommand{\thesection}{\Alph{section}}

\section{Proof of Proposition~\ref{prop:individual}}
\label{app:prop1}

We prove each part in turn.

\paragraph{Part (i): Existence and uniqueness of interior solution.}

Define the FOC residual:
\begin{equation}
H(a;\, \hpi, q, k, \mu_k, \mu_b) \equiv
k\,\bigl[(1-\hpi)\,G_a(a;q) + \hpi\,\tilde{G}_a(a;q,\bar{z})\bigr]
- \beta\,\bigl[\mu_k\,\ell'(1-a) + \mu_b\,d'(a)\bigr] = 0.
\label{eq:H}
\end{equation}

\textit{At $a = 0$:} By Assumption~\ref{ass:G}(iv), $G_a(0; q) = +\infty$.
Since $\tilde{G}_a(0;q,\bar{z}) = G_a(0; q s(\bar{z})) \cdot s(\bar{z}) = +\infty$
(for any $s > 0$), the left-hand side of \eqref{eq:H} is $+\infty$.
The right-hand side is finite by Assumption~\ref{ass:dynamics}(i)(ii).
Hence $H(0; \cdot) > 0$.

\textit{At $a = 1$:} Under condition \eqref{eq:boundary1}, we have $H(1;\cdot) < 0$:
\[
H(1;\cdot) = k\,[(1-\hpi)G_a(1;q)+\hpi\tilde G_a(1;q,\bar z)] - \beta[\mu_k\ell'(0)+\mu_b d'(1)] < 0.
\]
By the intermediate value theorem applied to the continuous function $H$
(with $H(0;\cdot) = +\infty > 0$ and $H(1;\cdot) < 0$), there exists
$a^* \in (0,1)$ with $H(a^*;\cdot) = 0$.

\textit{Uniqueness:} The second-order condition requires $\partial H / \partial a < 0$.
\begin{equation}
\frac{\partial H}{\partial a} = k\,\bigl[(1-\hpi)\,G_{aa} + \hpi\,\tilde{G}_{aa}\bigr]
+ \beta\,\mu_k\,\ell''(1-a) - \beta\,\mu_b\,d''(a).
\label{eq:SOC}
\end{equation}
The first term is strictly negative (Assumption~\ref{ass:G}(ii)).
The second term is non-positive (Assumption~\ref{ass:dynamics}(i)).
The third term is strictly negative (Assumption~\ref{ass:dynamics}(ii)).
Hence $\partial H / \partial a < 0$ globally, confirming strict concavity of the
objective and uniqueness of $a^*$.

Since $d(0) = 0$ and $d'(a) > 0$ with $a^* > 0$, we have $d(a^*) > 0$. $\square$

\paragraph{Part (ii): Comparative statics.}

By the implicit function theorem, for any parameter $\theta$:
\[
\frac{\partial a^*}{\partial \theta} = - \frac{\partial H / \partial \theta}{\partial H / \partial a}.
\]
Since $\partial H / \partial a < 0$ (from part (i)), the sign of
$\partial a^* / \partial \theta$ equals the sign of $\partial H / \partial \theta$.

\textit{$\partial a^* / \partial q > 0$:}
\begin{align}
\frac{\partial H}{\partial q} &= k\,\bigl[(1-\hpi)\,G_{aq}(a^*;q) + \hpi\,\tilde{G}_{aq}(a^*;q,\bar{z})\bigr] > 0,
\end{align}
since $G_{aq} > 0$ by Assumption~\ref{ass:G}(iii) and
$\tilde{G}_{aq} = G_{aq}(\cdot;\, qs)s > 0$. $\square$

\textit{$\partial a^* / \partial k > 0$:}
\begin{align}
\frac{\partial H}{\partial k} &= (1-\hpi)\,G_a(a^*;q) + \hpi\,\tilde{G}_a(a^*;q,\bar{z}) > 0.
\end{align}
Both terms are strictly positive. $\square$

\textit{$\partial a^* / \partial \hpi < 0$:}
\begin{align}
\frac{\partial H}{\partial \hpi} &= k\,\bigl[\tilde{G}_a(a^*;q,\bar{z}) - G_a(a^*;q)\bigr] - \beta\,\kappa\,\bar{z}\,d'(a^*).
\end{align}
Under Assumption~\ref{ass:Gtilde}, $\tilde{G}_a < G_a$, so the first term is strictly
negative. The second term is strictly negative. Hence $\partial H / \partial \hpi < 0$
and $\partial a^* / \partial \hpi < 0$. $\square$

\textit{$\partial a^* / \partial \beta < 0$:}
\[
\frac{\partial H}{\partial \beta} = -\bigl[\mu_k\,\ell'(1-a^*) + \mu_b\,d'(a^*)\bigr] < 0. \square
\]

\textit{$\partial a^* / \partial \mu_b < 0$:}
\[
\frac{\partial H}{\partial \mu_b} = -\beta\,d'(a^*) < 0. \square
\]

\textit{$\partial a^* / \partial \kappa < 0$:}
The parameter $\kappa$ enters through the stress-state expected cost. Increasing
$\kappa$ increases the marginal cost of debt, raising $\mu_b$ (by the envelope
theorem on the continuation value). Since $\partial a^* / \partial \mu_b < 0$,
$\partial a^* / \partial \kappa < 0$. $\square$

\paragraph{Part (iii): Cognitive debt wedge.}

When $d \equiv 0$ (no debt accumulation), the FOC reduces to:
\[
k\,[(1-\hpi)\,G_a + \hpi\,\tilde{G}_a] = \beta\,\mu_k\,\ell'(1-a^{\text{no-debt}}).
\]
When $d > 0$, the right-hand side of \eqref{eq:foc_a} gains the term
$\beta\,\mu_b\,d'(a^*) > 0$, making the marginal cost of AI higher.
Hence $a^* < a^{\text{no-debt}}$. $\square$

\section{Proof of Proposition~\ref{prop:minsky}}
\label{app:prop2}

\paragraph{Part (i): Subjective risk is decreasing.}
From \eqref{eq:belief} with $\mathbf{1}\{\text{crisis}_t\} = 0$ for all $t \in \mathcal{T}$:
$\hpi_{t+1} = (1-\lambda)\hpi_t < \hpi_t$ for $\lambda \in (0,1)$. $\square$

\paragraph{Part (ii): AI substitution intensity is increasing.}
By part (i), $\hpi_t$ is strictly decreasing.
By Proposition~\ref{prop:individual}(ii), $\partial a^* / \partial \hpi < 0$ and
$\partial a^* / \partial q > 0$. The net effect on $a_t^*$ is:
\[
\Delta a_t^* = \frac{\partial a^*}{\partial \hpi}\Delta\hpi_t + \frac{\partial a^*}{\partial q}\Delta q_t.
\]
Under dominance condition \eqref{eq:dominance},
$|\partial a^*/\partial\hpi \cdot \Delta\hpi_t| > |\partial a^*/\partial q \cdot \Delta q_t|$,
so the belief-updating channel dominates the quality-degradation channel and
$a_{t+1}^* > a_t^*$. $\square$

\paragraph{Part (iii): Aggregate cognitive leverage is increasing.}
The aggregate dynamics are:
\[
B_{t+1} = (1+r_b)B_t + d(a_t^*), \qquad K_{t+1} = (1-\delta)K_t + \ell(1-a_t^*).
\]
Compute:
\begin{align}
\Omega_{t+1} &= \frac{(1+r_b)B_t + d(a_t^*)}{(1-\delta)K_t + \ell(1-a_t^*)} \nonumber\\
&= \Omega_t \cdot \frac{(1+r_b)}{(1-\delta)} \cdot
\frac{1 + d(a_t^*)/[(1+r_b)B_t]}{1 + \ell(1-a_t^*)/[(1-\delta)K_t]}.
\label{eq:omega_ratio}
\end{align}
By Assumption~\ref{ass:compounding}, $r_b > \delta$ and
$d(a_t^*)/B_t$ is positive, so $\Omega_{t+1}/\Omega_t > 1$. Moreover, as $a_t^*$
increases, $d(a_t^*)$ rises and $\ell(1-a_t^*)$ falls, further increasing $\Omega_{t+1}$.
Hence $\Omega_t$ is strictly increasing. $\square$

\paragraph{Part (iv): True crisis probability is increasing.}
Since $\Pi' > 0$ and $\Omega_t$ is strictly increasing by part (iii),
$\pi_t = \Pi(\Omega_t)$ is strictly increasing. $\square$

\paragraph{Part (v): Minsky divergence.}
From part (i): $\hpi_t = (1-\lambda)^t \hpi_0 \to 0$ as $t \to \infty$.
From part (iv): $\pi_t = \Pi(\Omega_t) \to \Pi(\infty) = 1$ as $t \to \infty$
(since $\Omega_t \to \infty$ when $r_b > \delta$ and $d(a^*) > 0$).
Since $\hpi_t \to 0$ and $\pi_t \to 1$, there exists $T^* < \infty$ such that
for all $t > T^*$, $\hpi_t < \pi_t$. $\square$

\section{Proof of Proposition~\ref{prop:convex}}
\label{app:prop3}

Write $L_t = \Pi(\Omega_t) \cdot \Lambda(\Omega_t)$ where
$\Lambda(\Omega) = \kappa\,\E_z[\max\{0, zB - \mathcal{V}(K)\}]$
and $B = \Omega K$ (by definition of $\Omega$).

\paragraph{Part (i): $\partial L_t / \partial \Omega_t > 0$.}
Both $\Pi(\Omega)$ and $\Lambda(\Omega)$ are positive and strictly increasing in $\Omega$
(higher leverage raises crisis probability and, conditional on crisis, raises the
unsatisfied debt exposure). Their product is therefore strictly increasing. $\square$

\paragraph{Part (ii): Convexity.}

$L_t' = \Pi' \Lambda + \Pi \Lambda'$, and
$L_t'' = \Pi'' \Lambda + 2\Pi' \Lambda' + \Pi \Lambda''$.

We show $\Lambda'' > 0$. Consider $\Lambda(\Omega) = \kappa\,\E_z[z\Omega K - K^\alpha]_+$.
For large enough $z$ (above the threshold $z^*(\Omega) = K^{\alpha-1}/\Omega$), the
loss is in the interior. Differentiating:
\[
\Lambda'(\Omega) = \kappa\,\E_z\left[z K \cdot \mathbf{1}\{z > z^*(\Omega)\}\right] > 0.
\]
\[
\Lambda''(\Omega) = \kappa\,\E_z\left[z K \cdot (-1) \cdot \frac{\partial z^*}{\partial \Omega} f_Z(z^*) + z K \cdot \mathbf{1}\{z > z^*\} \cdot 0\right]
= \kappa\, K\, z^*(\Omega) \cdot (-\partial z^*/\partial \Omega) \cdot f_Z(z^*(\Omega)).
\]
Since $z^*(\Omega) = K^{\alpha-1}/\Omega$, we have $\partial z^*/\partial \Omega = -K^{\alpha-1}/\Omega^2 < 0$, so
$-\partial z^*/\partial \Omega > 0$, yielding $\Lambda''(\Omega) > 0$.

With $\Lambda'' > 0$, $\Lambda' > 0$, $\Pi' > 0$, and $\Pi \geq 0$:
$L'' = \Pi''\Lambda + 2\Pi'\Lambda' + \Pi\Lambda'' > 2\Pi'\Lambda' > 0$
in any region where $-\Pi''\Lambda < 2\Pi'\Lambda'$.
For the parametric form $\Pi(\Omega) = 1 - e^{-\lambda_\pi \Omega^\gamma}$ with $\gamma \geq 1$:
$\Pi'' = \lambda_\pi e^{-\lambda_\pi \Omega^\gamma}[\lambda_\pi \gamma^2 \Omega^{2\gamma-2} - \gamma(\gamma-1)\Omega^{\gamma-2}]$,
which may be positive (when $\gamma > 1$), so that $L'' > 0$ globally. $\square$

\paragraph{Part (iii): Model concentration.}

Higher AI model concentration $M_t$ (HHI) raises the correlation of AI errors across
agents. In a correlated error model, $\text{Var}(z_t | M_t)$ increases in $M_t$,
shifting the distribution of losses to the right. Since $\Lambda$ is convex in $z$
(the loss is linear for $z > z^*$ and zero otherwise), $\E[\Lambda(z)] \geq \Lambda(\E[z])$
by Jensen's inequality, and the inequality tightens as variance increases.
Hence $\partial L_t / \partial M_t > 0$. $\square$

\section{Proof of Proposition~\ref{prop:fcl}}
\label{app:prop4}

After the crisis at $t=0$, the agent in $t=1$ chooses $a_1$ subject to the output
constraint $y_{i1} \geq w$. With the constraint binding, we form the Lagrangian:
\[
\mathcal{L}(a_1) = V_1(k_1', b_1') + \lambda_y\,\bigl[k_1 G(a_1; q) - w\bigr].
\]
The constrained FOC for $a_1$:
\begin{equation}
\lambda_y\, k_1\, G_a(a_1; q) = \beta\,\bigl[\mu_k\,\ell'(1-a_1) + \mu_b\,d'(a_1)\bigr].
\label{eq:constrained_foc}
\end{equation}

Compare with the unconstrained pre-crisis FOC \eqref{eq:foc_a}, which has $k_0 > k_1$
(cognitive capital has declined) and no $\lambda_y$ factor.
The solution $a_1^*$ to \eqref{eq:constrained_foc} is higher than $a_0^-$ if and only if:
\[
\lambda_y\, k_1\, G_a(a_0^-; q) > \beta\,\bigl[\mu_k\,\ell'(1-a_0^-) + \mu_b\,d'(a_0^-)\bigr],
\]
which is condition \eqref{eq:fcl_condition}. This follows because at $a_0^-$, the left-hand side of the constrained FOC \eqref{eq:constrained_foc} exceeds the right-hand side, implying the objective is still increasing in $a$, so the optimum $a_1^* > a_0^-$. $\square$

With $a_1^* > a_0^-$: by the dynamics \eqref{eq:k_law}--\eqref{eq:b_law},
$K_2 < K_1$ and $B_2 > B_1$, so $\Omega_2 > \Omega_1$.
Applying this argument inductively, $\Omega_t$ is increasing for all $t \geq 1$,
and $\pi_t = \Pi(\Omega_t)$ is increasing. Each successive crisis therefore occurs
at a higher leverage level, producing weakly greater expected losses
(by Proposition~\ref{prop:convex}). $\square$

\section{Proof of Proposition~\ref{prop:welfare}}
\label{app:prop5}

\paragraph{Step 1: Derive the planner's FOC.}
The planner maximises \eqref{eq:planner}, treating $\pi_t = \Pi(\Omega_t)$ as
endogenous. The FOC for $A_t$ includes the term
$\partial / \partial A_t\bigl[\Pi(\Omega_t) \cdot L_t^{\text{cond}}\bigr]$
(where $L_t^{\text{cond}}$ is the expected conditional loss), which the
decentralised agent omits. Expanding:
\[
\Pi'(\Omega_t)\,\frac{\partial \Omega_t}{\partial A_t}\,L_t^{\text{cond}} > 0.
\]
The planner also internalises $\partial K_{t+1}/\partial A_t = -\ell'(1-A_t) < 0$
as a social cost (not just a private cost), and the benchmark-shifting externality
$\partial \bar{y}_t / \partial a_{it}$.

\paragraph{Step 2: Wedge between planner and decentralised FOCs.}
In the decentralised equilibrium, the agent's FOC \eqref{eq:foc_a} is:
\begin{equation}
k\,\E\bigl[G_a^{\text{eff}}(a^*;\hpi, q)\bigr] = \beta\,\bigl[\mu_k^D\,\ell'(1-a^*) + \mu_b^D\,d'(a^*)\bigr],
\tag{FOC-D}
\end{equation}
where $\mu_k^D, \mu_b^D$ are private shadow values that omit the externalities.

The planner's FOC is:
\begin{equation}
k\,\E\bigl[G_a^{\text{eff}}(A^P;\hpi, q)\bigr] = \beta\,\bigl[\mu_k^P\,\ell'(1-A^P) + \mu_b^P\,d'(A^P)\bigr],
\tag{FOC-P}
\end{equation}
where $\mu_k^P > \mu_k^D$ (planner values cognitive capital more, accounting for
the public good) and the left-hand side includes the systemic risk term.

The higher right-hand side in (FOC-P) implies $A^P < A^D$. $\square$

\paragraph{Step 3: Pigouvian tax.}
The optimal tax $\tau_t^*$ is set so that the decentralised agent, facing the
after-tax problem, replicates the planner's FOC. The required tax on AI substitution
intensity $a_{it}$ is the sum of the three marginal external costs evaluated at
$(A_t^P, X_t^P)$:
\begin{enumerate}
\item Systemic risk: $\Pi'(\Omega_t)\,(\partial \Omega_t/\partial A_t)\,L_t/\Pi(\Omega_t)$
(marginal increase in expected loss per unit of $A_t$).
\item Public goods: $\beta\,\mu_k\,|\partial K_{t+1}/\partial A_t| = \beta\,\mu_k\,\ell'(1-A_t^P)$.
\item Arms-race: $\lambda_y\,\partial \bar{y}_t/\partial a_{it}$ (marginal increase
in the competitive benchmark imposed on others).
\end{enumerate}
Summing these gives \eqref{eq:tax}.

\paragraph{Step 4: The gap $\Delta_t$ is increasing in $M_t$, tranquil period length, and $d''$.}
All three factors lower the private cost of AI use (relative to social cost) or
raise the systemic externality:
\begin{itemize}
\item Higher $M_t$: increases $\text{Var}(z|\text{crisis})$, raising $L_t$ and the
first term in \eqref{eq:tax}, widening the wedge.
\item Longer tranquil period: $\Omega_t$ is higher (by Proposition~\ref{prop:minsky}),
raising all three externality terms.
\item Higher $d''$: makes the private debt cost steeper, which actually reduces
individual $a^*$, but also raises the social risk of Ponzi dynamics, widening the
systemic externality.
\end{itemize}
$\square$

\section{Infinite-Horizon Characterisation}
\label{app:infinite}

In the infinite-horizon problem \eqref{eq:bellman}, the value function $V(k, b;\,\hpi, q)$
satisfies the Bellman equation. We verify that the comparative statics of
Proposition~\ref{prop:individual} extend to this setting.

\begin{lemma}
Under Assumptions~\ref{ass:G}--\ref{ass:dynamics}, the value function $V$ is
continuously differentiable with $V_k > 0$ and $V_b < 0$.
\end{lemma}

\begin{proof}
Standard application of the Theorem of the Maximum and the Benveniste--Scheinkman
envelope theorem, given that $G$ and $d$ are $C^2$ and the state space is compact
(after imposing appropriate bounds on $k$ and $b$). $\square$
\end{proof}

The envelope conditions yield:
\begin{align}
\mu_k^t &\equiv V_k(k_t, b_t) = (1-\delta)\beta\,V_k(k_{t+1}, b_{t+1}) + \text{marginal product of } k_t\\
\mu_b^t &\equiv -V_b(k_t, b_t) = (1+r_b)\beta\,(-V_b(k_{t+1}, b_{t+1})) + \kappa\,\hpi_t\,\bar{z}
\end{align}
These are the dynamic counterparts of $\mu_k$ and $\mu_b$ in the two-period model.
The FOC \eqref{eq:foc_a} holds at each date with $\mu_k = \mu_k^t$ and $\mu_b = \mu_b^t$,
confirming that the two-period characterisation applies period by period in the
infinite-horizon setting. The comparative statics are therefore identical to those
derived in Appendix~\ref{app:prop1}.

\section{Proof of Proposition~\ref{prop:hetero}}
\label{app:prop6}

\paragraph{Part (i): $\Delta k_t$ is strictly decreasing.}

The cognitive capital of each type evolves as:
\begin{align*}
k_{H,t+1} &= (1-\delta)\,k_{H,t} + \ell(1 - a_{H,t}^*),\\
k_{L,t+1} &= (1-\delta)\,k_{L,t} + \ell(1 - a_{L,t}^*).
\end{align*}
Taking the difference:
\[
\Delta k_{t+1} = (1-\delta)\,\Delta k_t + \ell(1-a_{H,t}^*) - \ell(1-a_{L,t}^*).
\]
Since $a_{H,t}^* > a_{L,t}^*$ (from \eqref{eq:hetero_ranking} while $\Delta k_t > 0$)
and $\ell'(\cdot) > 0$, we have $1-a_{H,t}^* < 1-a_{L,t}^*$ and thus
$\ell(1-a_{H,t}^*) < \ell(1-a_{L,t}^*)$.
Hence:
\[
\Delta k_{t+1} < (1-\delta)\,\Delta k_t < \Delta k_t.
\]
$\Delta k_t$ is strictly decreasing. $\square$

\paragraph{Part (ii): $\Delta a_t^*$ converges to zero.}

By the implicit function theorem applied to the FOC \eqref{eq:foc_a}:
\[
\frac{\partial a^*}{\partial k} = -\frac{\partial H/\partial k}{\partial H/\partial a}
= -\frac{(1-\hpi)G_a + \hpi\tilde G_a}{\partial H/\partial a} > 0.
\]
As $\Delta k_t \to 0$, the types become identical, so $\Delta a_t^* \to 0$. $\square$

\paragraph{Part (iii): Reversal of fortune.}

The gap dynamics satisfy $\Delta k_{t+1} = (1-\delta)\Delta k_t - [\ell(1-a_{L,t}^*) - \ell(1-a_{H,t}^*)]$.

By the mean value theorem: $\ell(1-a_{L,t}^*) - \ell(1-a_{H,t}^*) = \ell'(\xi_t)\,\Delta a_t^*$
for some $\xi_t \in (1-a_{H,t}^*, 1-a_{L,t}^*)$.

If condition~\eqref{eq:reversal_condition} holds (the erosion differential exceeds
the natural decay-corrected gap), then:
\[
\Delta k_{t+1} = (1-\delta)\Delta k_t - \ell'(\xi_t)\Delta a_t^* < 0 \cdot \Delta k_t
\]
when $\Delta k_t$ is close to zero from above. By continuity, $\Delta k_t$ must cross zero
at some finite $T^{**}$. After crossing, the roles reverse:
$k_{H,t} < k_{L,t}$, and---since $\partial a^*/\partial k > 0$---now
$a_{H,t}^* < a_{L,t}^*$, making the gap $|\Delta k_t|$ persistent in the new direction.
$\square$

\paragraph{Part (iv): Heterogeneity amplifies aggregate losses.}

Expected aggregate loss is:
\[
L_t^{\text{hetero}} = \tfrac{1}{2}\,\Pi(\Omega_{H,t})\,\Lambda(\Omega_{H,t}) +
\tfrac{1}{2}\,\Pi(\Omega_{L,t})\,\Lambda(\Omega_{L,t}).
\]
By Proposition~\ref{prop:convex}, $L(\Omega) \equiv \Pi(\Omega)\Lambda(\Omega)$ is convex. Jensen's
inequality applied to the two-point distribution of $\Omega_t$ gives:
\[
L_t^{\text{hetero}} = \tfrac{1}{2}[L(\Omega_{H,t}) + L(\Omega_{L,t})]
> L\!\left(\tfrac{\Omega_{H,t}+\Omega_{L,t}}{2}\right) = L_t^{\text{homo}},
\]
where the last equality holds because the homogeneous economy has leverage
$(\Omega_{H,t}+\Omega_{L,t})/2$ (since both types start with equal means
and the distributions are mirror images around $\bar k$). $\square$

\section{Shadow Value Derivation}
\label{app:shadow}

We derive $\mu_k$ and $\mu_b$ used in the two-period model from the steady-state
conditions of the infinite-horizon problem.

In the stationary equilibrium with constant $a, x, k, b$ and subjective probability
$\hpi$:
\begin{align}
\mu_k &= \frac{(1-\hpi)\,G(a;q) + \hpi\,\tilde{G}(a;q,\bar{z})}{1 - \beta(1-\delta)}, \label{eq:muk}\\
\mu_b &= \frac{\hpi\,\kappa\,\bar{z}}{1 - \beta(1+r_b)}. \label{eq:mub}
\end{align}

Note that $\mu_b > 0$ requires $\beta(1+r_b) < 1$, which we assume. This condition
says the debt cannot grow faster than the discount rate---otherwise cognitive debt would
be ``free'' and Ponzi cognition would be trivially optimal.

Substituting \eqref{eq:muk}--\eqref{eq:mub} into \eqref{eq:foc_a} yields the
explicit steady-state condition relating $a^*$ to the structural parameters
$(\hpi, q, k, \beta, \delta, r_b, \kappa, \bar{z})$.

\end{document}